\documentclass{article}

\usepackage[utf8]{inputenc} 
\usepackage[T1]{fontenc}    
\usepackage{hyperref}       
\usepackage{url}            
\usepackage{booktabs}       
\usepackage{amsfonts}       
\usepackage{nicefrac}       
\usepackage{microtype}      
\usepackage{lipsum}
\usepackage{fancyhdr}       
\usepackage{graphicx}

\usepackage{hyperref} 
\usepackage{arxiv}
\usepackage{amsmath,amsfonts, mathtools, bm}
\usepackage{amsthm,amssymb}
\usepackage[dvipsnames]{xcolor}
\usepackage{colortbl}
\usepackage{cite}
\usepackage{enumitem}
\usepackage[ruled,noresetcount,vlined,linesnumbered]{algorithm2e}
\SetKwInput{KwInput}{Input}                
\SetKwInput{KwOutput}{Output} 
\usepackage{makecell}
\usepackage{tikz}
\usepackage{pifont}
\usepackage{pgfplots}
\usepackage{booktabs, enumitem}
\pgfplotsset{compat=1.7}
\usepackage{bbm}
\usepackage{subcaption} 
\usetikzlibrary{arrows,matrix, positioning, shapes.geometric}
\usepackage{listofitems} 
\usepackage{nicefrac}
\usepackage{xcolor}

\newtheorem{proposition}{Proposition}

\newcommand{\bluetext}[1]{\textcolor{blue}{#1}}

\newcommand{\redtext}[1]{\textcolor{red}{#1}}
 
\DeclareMathOperator*{\argmax}{argmax}

\makeatletter\def\Hy@xspace@end{}\makeatother

\newcommand{\cC}{\mathcal{C}} 
 
\newcommand{\cG}{\mathcal{G}} 
 
\newcommand{\cM}{\mathcal{M}} 
\newcommand{\cO}{\mathcal{O}} \newcommand{\cP}{\mathcal{P}}
 \newcommand{\cS}{\mathcal{S}}

 \newcommand{\cX}{\mathcal{X}}

\DeclareMathOperator{\proj}{proj}

\usepackage{esvect}

\title{End-to-End Optimization and Learning of Fair Court Schedules}

\author{
My H.~Dinh \\
	University of Virginia\\
	Charlottesville, VA, USA\\
	\texttt{fqw2tz@virginia.edu}\\
	\And
    James Kotary \\
	University of Virginia\\
	Charlottesville, VA, USA\\
	\texttt{jk4pn@virginia.edu} \\
     \And
 Lauryn P. Gouldin \\
  Syracuse University \\
  Syracuse, NY 13210 \\
  \texttt{lgouldin@syr.edu} \\  
  \And
 William Yeoh \\
 Washington University in St. Louis \\
 Saint Louis, MO 63130 \\
  \texttt{wyeoh@wustl.edu} \\
  \And
  Ferdinando Fioretto \\
	University of Virginia\\
	Charlottesville, VA, USA\\
	\texttt{fioretto@virginia.edu} \\
}

\begin{document}
\maketitle
\begin{abstract}
Criminal courts across the United States handle millions of cases every year, and the scheduling of those cases must accommodate a diverse set of constraints, including the preferences and availability of courts, prosecutors, and defense teams. When criminal court schedules are formed, defendants' scheduling preferences often take the least priority, although defendants may face significant consequences (including arrest or detention) for missed court dates. Additionally, studies indicate that defendants' nonappearances impose costs on the courts and other system stakeholders. To address these issues, courts and commentators have begun to recognize that pretrial outcomes for defendants and for the system would be improved with greater attention to court processes, including \emph{court scheduling practices}. There is thus a need for fair criminal court pretrial scheduling systems that account for defendants' preferences and availability, but the collection of such data poses logistical challenges. Furthermore, optimizing schedules fairly across various parties' preferences is a complex optimization problem, even when such data is available. In an effort to construct such a fair scheduling system under data uncertainty, this paper proposes a joint optimization and learning framework that combines machine learning models trained end-to-end with efficient matching algorithms. This framework aims to produce court scheduling schedules that optimize a principled measure of fairness, balancing the availability and preferences of all parties. 
\end{abstract}


\section{Introduction}
\label{sec:introduction}
Criminal courts across the United States handle millions of cases every year, and scheduling these cases must accommodate a diverse set of constraints, including the preferences and availability of courts, prosecutors, and defense teams. Typically, defendants’ scheduling preferences receive the least priority, despite the significant consequences they may face---such as arrest, detention, or other penalties---for missed court dates.

Ensuring that defendants return to court for trial and all pretrial appearances has shaped pretrial decision-making for centuries. Pretrial appearance rates vary across jurisdictions, but studies suggest that most defendants appear for their pretrial hearings as required \cite{gouldin2024keeping, cjil2022northcarolina, reaves2013felony}. Nevertheless, criminal defendants still fail to appear for court in significant numbers. These nonappearances impose costs on the courts and other system stakeholders. Demographic factors like race, gender, or age have been shown in prior studies to impact nonappearance rates, but those impacts are not consistent across studies \cite{zettler2015exploratory}. The factor most closely and consistently associated with nonappearance--across other demographic factors--is indigence \cite{zettler2015exploratory}. 

In the last decade, bail reform efforts across the country have led to higher rates of pretrial release and increased focus on maintaining or even improving defendants’ pretrial appearance rates. Judges and court administrators are concentrating more on strategies to ensure that defendants return for their court appearances. Many of these reforms aim to change defendants’ behavior, but courts and commentators are beginning to recognize that pretrial outcomes for both defendants and the system could be improved by giving greater attention to court expectations and processes. To address pretrial appearance issues, courts must investigate institutional reforms, including changes to how pretrial court appearances are scheduled.

\begin{figure*}[t]
\center
\includegraphics[trim={10em 30em 10em 25em}, clip, width=1.0\columnwidth]{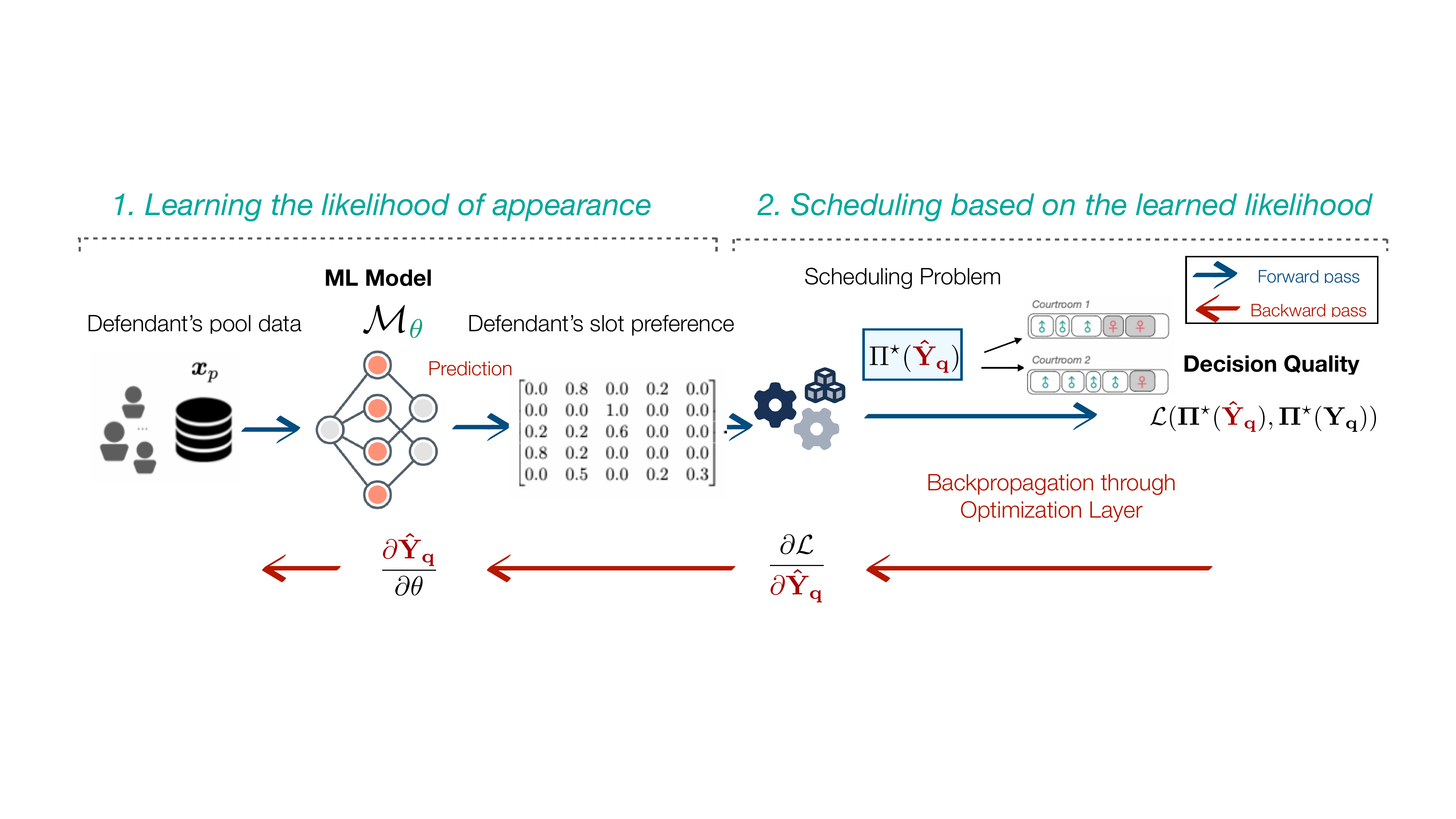}
\caption{Proposed end-to-end framework for learning to schedule. Given candidates' socioeconomic and demographic identifiers, a neural network is trained to predict their preference score for each time slot. A differentiable surrogate model uses a predicted score to attain assignments and decision quality loss. }
\label{fig:court_scheme}
\end{figure*}

Across the system, there is increasing awareness of the conflict between many courts' current inflexible pretrial scheduling processes and the particular challenges that face defendants who have insecure employment situations, who have care-giving responsibilities, or who lack transportation to court \cite{ideas42_2023}. These challenges are disproportionately borne by poor and minority defendants, leading to disparate impacts with profound negative effects.

Additionally, nonappearance imposes high costs on the judiciary as an institution. Rescheduling court dates imposes costs on judges, attorneys, court personnel, witnesses, and other defendants whose cases may be delayed as a result of court calendar congestion. Pretrial scheduling has thus a significant societal and economic impact on the involved individuals and for society at large.

There is a need for fair criminal court pretrial scheduling systems that account for defendants’ preferences and availability, but collecting such data is logistically challenging. Furthermore, fairly optimizing a schedule across various parties’ preferences is a complex optimization problem, even when the necessary data is available. This paper proposes a joint optimization and learning framework that combines machine learning models trained end-to-end with efficient matching algorithms. This framework aims to produce court scheduling policies that optimize a principled measure of fairness, balancing the availability and preferences of all involved parties.

Designing court schedules that account for the diverse constraints of all parties involved presents several challenges. While it is feasible to gather availability preferences from judges and attorneys, the collection of such data from defendants is often hindered by logistical and financial barriers, as well as inequities in access to communication technologies (e.g., email, phones). As a result, court scheduling systems must operate under conditions of \emph{uncertainty}. Furthermore, even when availability data is accessible, there is a lack of formal mechanisms to ensure that scheduling processes adhere to principles of procedural fairness. Current court scheduling systems do not adequately incorporate fairness considerations, which are essential for promoting justice. Moreover, the development of fairness-aware scheduling models is inherently challenging, requiring the integration of complex constraints into machine learning models. Compounding this difficulty is the absence of comprehensive datasets needed to conduct and validate such studies.

To address these challenges, there is a clear need for a system that reduces nonappearance while ensuring fairness for all defendants.  However, addressing these issues is far from trivial. Scheduling is fundamentally a decision-making process, typically framed as a utility maximization problem. The goal is to assign people to time slots in a way that fairly maximizes overall utility. In court scheduling, this utility can be thought of as a function of an individual's preference, even though such data may be unavailable or uncertain. The task then becomes learning how to map defendants' socioeconomic and demographic characteristics to court schedules that meet fairness criteria. In this paper, we propose a novel framework that combines combinatorial optimization with deep learning to compose an efficient court scheduling system that maximizes a principled notion of fairness under uncertainty in the defendants' scheduling preferences. A schematic illustration of the overall framework is provided in Figure \ref{fig:court_scheme}.

\medskip\noindent{\bf Contributions.}
This paper makes the following contributions:
\begin{enumerate}[left=4pt,topsep=4pt]
    \item It formalizes a fair scheduling problem tailored for pretrial court settings, where data about defendants’ scheduling preferences is often incomplete or uncertain. The problem is formulated as a nonlinear integer program with a piecewise-defined fairness objective, over uncertain values in defendants' scheduling preferences. This formulation addresses the real-world challenges courts face in obtaining reliable preference data from defendants.

    \item It presents a novel framework that integrates machine learning with optimization techniques, allowing to handle uncertainty in defendants' preferences by directly optimizing a decision loss function, thereby improving the robustness and fairness of the scheduling outcomes.

    \item To overcome the computational complexity of the integer programming formulation, it proposes an end-to-end trainable model that predicts feasible schedules via the parameters of an efficient matching algorithm. This enables training on the fairness objective directly, enhancing scalability as well as satisfaction of the fairness objective.

    \item Given the absence of datasets on court scheduling-dependent preferences that include sensitive information such as gender, race, and childcare responsibilities, it introduces a process for generating realistic synthetic datasets for training. This process builds from available information about the demographic composition of arrestee populations and layers in additional estimates of the numbers of defendants in the pretrial population who face particular types of barriers to timely court appearances. This enables the training and validation of our fairness-aware scheduling models, ensuring that they can account for diverse and sensitive demographic factors. 

    
\end{enumerate}

\section{Related Work}
\label{sec:related_work}

\paragraph{Court scheduling.}
Court scheduling refers to the process of organizing the court’s calendar and scheduling cases at the appropriate time, location, and format for all necessary participants. It is a fundamental function of the judicial system and effective court scheduling takes into account available resources, participant schedules, and due process requirements \cite{lane1976guide, graef2023systemic}. Despite its critical role, however, court scheduling remains largely understudied \cite{gouldin2024keeping}. Most existing research focuses on the efficacy of court reminder systems or primarily aims to improve pretrial outcomes for individual defendants \cite{ideas42_2023, cjil2022northcarolina}.

This work addresses these gaps by evaluating the current processes used by court personnel to establish court schedules and proposing data-driven improvements. Numerous factors create barriers to innovation in this context: Courts adhere to tradition and to formal, hierarchical structures that hinder efforts to streamline court scheduling \cite{inslaw1988decision, ferguson2022courts}. Additionally, due process requirements and the adversarial nature of the criminal process pose significant barriers to change. The inherent complexity of the judicial process itself is a considerable hurdle to innovation \cite{inslaw1988decision}. Effective scheduling must account for multiple factors, including the availability of judges, defendants, attorneys, and courtrooms. The complex landscape leads to hesitancy among administrators, who may prefer maintaining the status quo over implementing potentially disruptive changes. Consequentially, the court system is marked by significant inefficiencies, causing delays and backlogs that can affect the delivery of justice. Moreover, the lack of data regarding court scheduling practices poses a major hurdle to reform efforts. Addressing these issues is crucial for developing a more efficient and effective court scheduling system.  

\paragraph{Advances in Machine Learning and Optimization for scheduling.}
Scheduling is fundamentally a combinatorial optimization problem. In the context of court scheduling, where defendants’ preferences are unknown a priori but essential for the optimization process, the scheduling problem typically relies on predictions from a machine learning model. This model forecasts defendants’ scheduling preferences based on predefined features such as the defendant’s job, access to transportation, and childcare obligations. This approach falls under the general area of constrained optimization informed by machine learning predictions, commonly referred to as a two-stage process or predictive and prescriptive processes.
However, this two-stage approach has been shown to lead to suboptimal outcomes. The primary issue arises from the misalignment between prediction errors and the final task utilities, which, in the case of this paper, account for the effectiveness and fairness of the court schedule. Since the prediction model is trained independently of the optimization task, inaccuracies in predicting defendants’ preferences can directly impact the quality of the resulting schedule. This separation means that the optimization process does not account for how prediction errors might affect the overall utility, leading to schedules that may not optimally balance fairness and efficiency.

Recent literature has sought to address these limitations by developing constrained optimization models that are trained end-to-end with machine learning models \cite{kotary2021end}. In the Predict-Then-Optimize setting, a machine learning model predicts the unknown coefficients of an optimization problem. Then, backpropagation through the optimal solution of the resulting problem allows for end-to-end training of its objective value, under ground-truth coefficients, as a loss function. The primary challenge is backpropagation through the optimization model, for which a variety of alternative techniques have been proposed. Differentiation through constrained argmin problems in the context of machine learning was discussed as early as \cite{gould2016differentiating}, who proposed first to implicitly differentiate the argmin of a smooth, unconstrained convex function by its first-order optimality conditions, defined when the gradient of the objective function equals zero. This technique is then extended to find approximate derivatives for constrained problems, by applying it to their unconstrained log-barrier approximations. Subsequent approaches applied implicit differentiation to the KKT optimality conditions of constrained problems directly \cite{amos2019optnet,amos2019limited}, but only on special problem classes such as  Quadratic Programs. \cite{konishi2021end} extend the method of \cite{amos2019optnet}, by modeling second-order derivatives of the optimization for training with gradient boosting methods. \cite{donti2017task} uses the differentiable quadratic programming solver of \cite{amos2019optnet} to approximately differentiate general convex programs through quadratic surrogate problems.   Other problem-specific approaches to analytical differentiation models include ones for sorting and ranking \cite{blondel2020fast}, linear programming \cite{mandi2020interior}, and convex cone programming \cite{agrawal2019differentiating}.

A related line of work concerns end-to-end learning with \emph{discrete} optimization problems, including linear, mixed-integer, and constraint programs. These problem classes often define discontinuous mappings with respect to their input parameters, making their true gradients unhelpful as descent directions in optimization. Accurate end-to-end training can be achieved by \emph{smoothing} the optimization mappings, to produce approximations that yield more useful gradients. A common approach is to augment the objective function with smooth regularizing terms such as Euclidean norm or entropy functions \cite{wilder2018melding,ferber2020mipaal,mandi2020interior}. Others show that similar effects can be produced by applying random noise to the objective \cite{berthet2020learning,paulus2020gradient}, or through finite difference approximations \cite{poganvcic2019differentiation,sekhar2022gradient}. This enables end-to-end learning with discrete structures such as constrained ranking policies \cite{kotary2022end}, shortest paths in graphs \cite{elmachtoub2020smart}, and various decision models \cite{wilder2019melding}.

This paper builds on the framework of end-to-end learning with discrete optimization problems, but unlike previous proposals, introduces a novel process for fair optimization based on ordered weighted average (OWA) operators. These operators enable the achievement of key fairness desiderata, including impartiality, equitability, and Pareto efficiency, which are essential for the court scheduling application examined. 

\section{Motivations and Problem Setting}
\label{sec:motivations}

Scheduling defendants in pretrial processes involves arranging their court appointments in a way that aligns with their preferences for various appointment times. In this context, a defendant’s preference indicates their likelihood of attending a scheduled appointment. A significant challenge is that these preference data are often unavailable, necessitating estimation from available data to create meaningful schedules. This challenge can be partially addressed by employing a pipeline where a machine learning (ML) model predicts defendants’ preferences, which are then used as inputs for the scheduling task, as illustrated in Figure \ref{fig:court_scheme}.

A common scheme in data-driven decision processes is to treat the learning and optimization components of this pipeline separately, where an ML model is first trained via a measure of predictive accuracy, and then its predictions are used to inform a decision-making task. 
Such 'two-stage' approaches are justified when the predictive models are near-perfect since accurate predictions tend to inform the correct decisions. However, in practice, predictive models are rarely perfect, especially when preference data are limited or incomplete. The resulting prediction errors can lead to suboptimal scheduling outcomes, where the generated schedules fail to fully align with the defendants’ actual preferences. Moreover, these prediction inaccuracies can exacerbate \emph{unfairness} in the scheduling process, as errors may disproportionately affect certain groups of defendants, leading to biased or inequitable scheduling outcomes \cite{kotary2021end}.
These effects are significant and, if not mitigated, can lead to a ``poor get poorer'' effect with significant societal consequences.  \emph{To mitigate these disparate impacts, this paper proposes an integrated optimization and learning framework for scheduling defendant court visit times to certify a desired fairness requirement.}

\subsection*{Problem Setting}
\label{sec:Problem_Setting}

The setting studied in this paper considers a training dataset $\bm{D} = (\bm{x}_p, \bm{g}_p, \bm{Y}_p)_{p=1}^N$ of $N$ elements, each describing a pool of $n$ defendants to be scheduled on a given day.
For each pool indexed by $p \in [N]$, $\bm{x}_p \in \cX$ describes a list $(x^p_i)_{i=1}^n$ of $n$ defendants to schedule, with each item $x^p_i$ 
defined by a feature vector. 
These feature vectors encode representations of the individuals to schedule, e.g., their socioeconomic identifiers, addresses, and accusations. 
The elements $\bm{g}_p = (g^p_i)_{i=1}^n$ describe protected group attributes of the defendants in some domain $\cG$. 
For example, they may denote employment type, whether the defendant has access to transportation, or whether they have childcare obligations. 
Together, the unprotected and protected attributes ($x_i^p, g_i^p$) provide a description of the defendant's $i$ characteristics in pool $p$. 
Finally, the element $\bf{Y}_p \in \mathcal{Y}  \subseteq \mathbb{R}^{n \times n}$ are supervision labels $(\bm{Y}^p_i)_{i=1}^n$ that associate a vector of non-negative values with each person $i \in [n]$, and describe the preference of individual $i$ with respect to each slot in a possible schedule. 

The goal is to learn a mapping between features $\bm{x}_p$ of a pool $p$ of $n$ defendants and a \emph{schedule} which fairly optimizes the satisfaction of their preferences as indicated by $\bm{Y}_{\bm{p}}$. A schedule is represented by a permutation of the list $[n]$, which determines the appointment times assigned to each defendant. When clear from the context, we may drop the various elements of the superscript or subscript $p$. For modeling purposes, and when clear from context, we represent a schedule $\bm{\Pi}$ by a permutation matrix, so that $\bm{\Pi}_{ij}$ indicates the assignment of individual $i$ to the time slot $j$. Additionally, let $\bm{\Pi}_i$ represent the $i^{\textit{th}}$ row of the matrix $\bm{\Pi}$.

Figure \ref{fig:scheduling-example} illustrates these elements in our problem setting. It shows a matrix  $\bm{\Pi}$ representing a permutation that matches defendants to appointment slots. The highlighted defendant's feature vector $x_4$, is illustrated in light of their respective row $\bm{\Pi}_4 = [1, 0,0,0,0]$ in $\bm{\Pi}$, which indicates that defendant $4$ is to be scheduled first. The contextual information provided by $x_4$ will be used as input to an ML model to produce an estimated preference vector $\bm{y}_4 = [0.8, 0.2, 0, 0, 0]$, which associates the defendant's likelihood of appearance to the different court visit times. The protected group information $\bm{g}$, expressing gender in the example above, is represented by the different colored (white/gray) cells. 

\begin{figure}[t]
    \centering
    \includegraphics[width=0.5\linewidth]{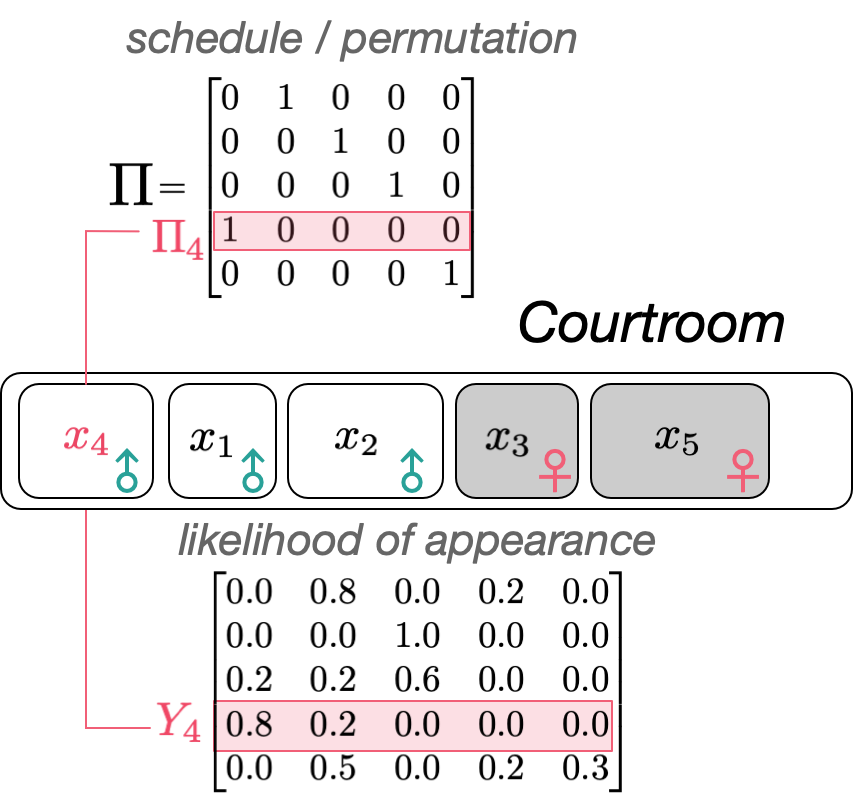}
    \caption{Court scheduling example}
    \label{fig:scheduling-example}
\end{figure}

\medskip
\paragraph{Learning Objective}
The goal in our setting is to predict, for any pool of defendants $(\bf{x},{g}, {Y})$, a permutation matrix $\bm{\Pi}$ representing a schedule which fairly maximizes the \emph{utility} of each defendant, defined as:
\begin{equation}
\label{eq:individual_utility}
    u_i(\bm{\Pi}, \bm{Y}) = Y_i^T \Pi_i,
\end{equation}
for defendant $i$. When $Y_{ij}$ represents the likelihood of defendant $i$ attending appointment $j$, this quantity represents the likelihood of defendant $i$ attending their assigned appointment. Given any $\bm{\Pi}$ and $\bm{Y}$, computing Equation \eqref{eq:individual_utility} for each defendant results in a \emph{utility vector} $\bm{u} \in \mathbb{R}^n$:
\begin{equation}
\label{eq:utility_vector}
    \bm{u}(\bm{\Pi}, \bm{Y}) = \textit{diag}(\bm{Y}^T \bm{\Pi}).  
\end{equation}
In this setting, it is natural to ask: 
\begin{quote}
    What does it mean to fairly optimize the utilities of each defendant, which are inherently independent objectives?     
\end{quote}
Any optimization method requires a single-valued objective function to be defined. For example, one may choose to optimize the total utility of a scheduling policy, defined as the sum of individual utilities:
\begin{equation}
\label{eq:utility}
    \bm{1}^T \bm{u}(\bm{\Pi}, \bm{Y}) =\mathrm{Tr}(\bm{Y}^T \bm{\Pi}), 
\end{equation}
which is a single-valued function of $\bm{u}$. However, a schedule $\bm{\Pi}$ which optimizes total utility does not account for the lowest utilities among all defendants, which may be arbitrarily low. This makes the total utility an undesirable objective in our framework, since \emph{accepting low utilities for some defendants increases their odds of nonappearance}.

Instead we seek an aggregation of individual utilities which, when optimized, raises the utilities of all defendants uniformly, to the extent possible. One such possible objective is to maximize the minimum of all defendants' utilities. However, this leads to \emph{pareto inefficient} solutions, meaning that some defendants' utilities are needlessly suboptimal, in the sense that $\bm{\Pi}$ can be chosen to raise those utilities without harm to the utilities of others \cite{ogryczak2003solving}. This can be viewed as a needless compromise to total utility. {\em The notion of court scheduling studied in this work calls for an aggregation of individual utilities which, when optimized, raises the lowest individual utilities while maintaining pareto efficiency}. 

A well-known aggregation function possessing these properties is the \emph{Ordered Weighted Average} (OWA)~\cite{Yager199380}. The OWA operator and its fairness properties are formally defined next, before delving into the description of the proposed framework for learning schedules which maximize the OWA-aggregated utility, in Section~\ref{sec:learning_with_OWA}.

\section{Preliminaries: Fair OWA Aggregation}
The $\textit{Ordered Weighted Average}$ (OWA) operator \cite{Yager199380} is a class of functions used for aggregating multiple independent values in settings requiring multicriteria evaluation and comparison \cite{10.5555/2464909}. Let $\mathbf{y} \in \mathbb{R}^m$ be a vector of $m$ distinct criteria, and $\tau: \mathbb{R}^m \rightarrow \mathbb{R}^m$ be the sorting map in increasing order so that $\tau_1(\mathbf{y}) \leq \tau_2(\mathbf{y}) \leq \cdots \leq \tau_m(\mathbf{y})$.

Then for any $\mathbf{w}$ satisfying 
$$\left\{ \mathbf{w} \in \mathbb{R}^m \,\vert\, \sum_i w_i = 1, \mathbf{w} \geq 0 \right\},$$ 
the OWA aggregation with weights $\mathbf{w}$ is defined as a linear functional on  $\tau(\mathbf{y})$:
\begin{equation}
\label{eq:OWA_definition}
\textsc{OWA}_{\mathbf{w}}(\mathbf{y}) = \mathbf{w}^T \tau(\mathbf{y}),
\end{equation}
which is piecewise-linear in $\mathbf{y}$ \cite{ogryczak2003solving}. 

This paper focuses on a specific instance of OWA, commonly known as \emph{Fair OWA} \cite{10.1155/2014/612018}, characterized by weights arranged in descending order: $w_1 > w_2 \ldots > w_n > 0$. Note that with monotonic weights, Fair OWA is also concave. Fair OWA objectives are increasingly popular in optimization as fairness gains attention in decision-making processes.

The following three properties of Fair OWA functions are crucial for their use in fairly optimizing multiple objectives:
\begin{enumerate}
    \item \emph{Impartiality} ensures that Fair OWA treats all criteria equally. This means that for any permutation $\sigma \in \cP_{m}$, where $\cP_{m}$ is the set of all permutations of $[1,\ldots,m]$, the OWA aggregation with weights $\mathbf{w}$ yields the same result for any permutation of the input vector $\mathbf{y}$.

    \item \emph{Equitability} guarantees that marginal transfers from a criterion with a higher value to one with a lower value increase the OWA aggregated value. This condition holds that $\textsc{OWA}_{\mathbf{w}}(\mathbf{y}_\epsilon) >  \textsc{OWA}_{\mathbf{w}}(\mathbf{y})$, where $\mathbf{y}_{\epsilon} = \mathbf{y}$ except at positions $i$ and $j$ where $(\mathbf{y}_{\epsilon})_i = \mathbf{y}_i - \epsilon$ and $(\mathbf{y}_{\epsilon})_j = \mathbf{y}_j + \epsilon$, assuming  $y_i > y_j +\epsilon$.  

    \item \emph{Monotonicity} ensures that $\textsc{OWA}_{\mathbf{w}}(\mathbf{y})$ is an increasing function of each element of $\mathbf{y}$. This property implies that solutions optimizing the OWA objectives \eqref{eq:OWA_definition} are Pareto Efficient solutions of the underlying multiobjective problem, thus no single criteria can be raised without reducing another \cite{ogryczak2003solving}. 
\end{enumerate}
This last aspect is crucial in optimization, where Pareto-efficient solutions are preferred over those that do not possess this attribute. Taken together, these properties define a notion of fairness in optimal solutions known as \emph{equitable efficiency} \cite{ogryczak2003solving}, which is of particular interest for the pretrial court scheduling optimization.

Intuitively, OWA objectives lead to fair optimal solutions by always assigning the highest weights of $\mathbf{w}$ to the objective criteria in order of lowest current value. This paper employs fair OWA functions as a fair measure of overall utility with respect to defendants’ preferences, ensuring that the resulting court schedules uphold principles of fairness and equity. As shown in the next section, while achieving these properties is desirable for court schedules, it also introduces a challenging optimization problem over the space of permutation matrices. Addressing such computational challenges is one of the key technical objectives of the paper.

\section{Fair Optimization of Court Schedules }
\label{sec:OWA_optimization}

Using the concepts introduced above, we can form an optimization program which models fair maximization of utilities in court schedules. Suppose a pool of defendants is described by $(\bm{x},\bm{g}, \bm{Y})$, where the preference values $\bm{Y}$ are known. We can model the scheduling matrix $\bm{\Pi}$ which maximizes the OWA-aggregated utility over all defendants, as the solution to an integer program:
\begin{subequations}
    \label{model:OWA_individual}
    \begin{align}
       \bm{\Pi}^{\star}(\bm{Y}) = {\argmax}_{\bm{\Pi}} &\;\; \textsc{OWA}_{\mathbf{w}} \left( \; \bm{u}(\bm{\Pi},\bm{Y}) \;\right) \\
       \mbox{subject to:} & \;\;\;\bm{\Pi} \in \{0,1\}^{n \times n} \label{model:OWA_individual_integers} \\
       & \;\; \sum_{i}  \Pi_{ij} =1,  \forall j \in [n] \label{model:OWA_individual_constraints_i}\\
       & \;\; \sum_{j} \Pi_{ij}=1, \forall i \in [n] \label{model:OWA_individual_constraints_j},
    \end{align}
\end{subequations}
\noindent where the utility vector $\bm{u}(\bm{\Pi},\bm{Y})$ is defined as in equation \eqref{eq:utility_vector}. Here and throughout, we make a standard choice of fair OWA weights $w_j = \frac{n-i+1}{n}$, known as the \emph{Gini indices} \cite{do2022optimizing} Note that the constraints \eqref{model:OWA_individual_constraints_i} and \eqref{model:OWA_individual_constraints_j} hold that each row and column of $\bm{\Pi}$ must sum to $1$. Together with \eqref{model:OWA_individual_integers}, this ensures a single value of $1$ in each row and column, so that problem \eqref{model:OWA_individual} models the permutation matrix with maximal fair OWA-aggregated utility.

\subsection{Group Fairness}
The objective in problem \eqref{model:OWA_individual} is to maximize the OWA-aggregated utility vector with respect to each individual defendant. We may also extend the use of the OWA operator to model different fairness objectives within our framework. In particular, we are interested in \emph{group fairness}, a concept widely employed, for example, in web search rankings \cite{singh2018fairness,singh2019policy,do2022optimizing,zehlike2020reducing,zehlike2017fa}. 

In the notation of Section \ref{sec:motivations}, each pool of defendants is described by data $(\bm{x}_p,\bm{g}_p, \bm{Y}_p)$, where $\bm{g}_p$ indicates the protected group of each defendant in the pool. Individuals may grouped by any category between which fair outcomes are desired: for  example by gender, race, socioeconomic status or intersections thereof. For any schedule $\bm{\Pi}$ we define the \emph{group utility} $u^g$ of group $g$ as the mean utility of all defendants in that group. For any group indicator $g$, let $\cS_{g}$ be the set of defendants' indices belonging to that group. Then 
\begin{equation}
\label{eq:group_utility}
    u^g(\bm{\Pi}, \bm{Y}) = \frac{1}{|\cS_{g}|} \sum_{i \in \cS_{g}} \bm{Y}_i^T \bm{\Pi}_i.  \\
\end{equation}
Let a \emph{partition} $\cG = \{ g_i^p: p \in [N], i \in [n] \}$ represent the set of all unique protected group indicators (e.g. male and female, or the set of all income brackets). Our notion of group fairness is to maximize the OWA aggregation of all group utilities over a chosen partition. Letting $\bm{u}^{\cG}$ be the vector of all group utilities $[u^g: g \in \cG]$,
\begin{subequations}
    \label{model:OWA_group}
    \begin{align}
       \bm{\Pi}^{\star}(\bm{Y}) = {\argmax}_{\bm{\Pi}} &\;\; \textsc{OWA}_{\mathbf{w}} \left( \; \bm{u}^{\cG}(\bm{\Pi},\bm{Y}) \;\right) \label{model:OWA_group_obj} \\
       \mbox{subject to:} & \;\;\;\bm{\Pi} \in \left\{0,1\right\}^{n \times n} \label{model:OWA_group_individual_integers} \\
       & \;\; \sum_{i}  \Pi_{ij} =1, \forall j \in [n] \label{model:OWA_group_individual_constraints_i} \\
       & \;\; \sum_{j} \Pi_{ij}=1, \forall i \in [n]. \label{model:OWA_group_individual_constraints_j}
    \end{align}
\end{subequations}

Of course, the models \eqref{model:OWA_group} and     \eqref{model:OWA_individual} coincide when each individual defendant constitutes their own group. In Section \ref{sec:Experiments}, we will evaluate the ability of our framework to fairly optimize group utilities with respect to various partitions. 

\subsection{Complexity of the Optimization Models}
It is important to remark on the complexity of the integer program  \eqref{model:OWA_group}. Since the OWA function is piecewise linear \cite{ogryczak2003solving}, they can be categorized as \emph{nonlinear integer programs} with a piecewise-defined objective and is thus NP-hard. 
Given the form of this integer program, traditional integer programming approaches cannot be applied directly to  \eqref{model:OWA_group}. 
Specifically, methods such as branch-and-bound and cutting plane algorithms, which are commonly used for solving integer linear programs (ILPs), struggle with the piecewise linear nature of the OWA objective. These ILP approaches typically rely on linearity and convexity to efficiently explore the solution space, but the nonlinear, piecewise-defined objective complicates the feasible region and increases the computational burden. 

Furthermore, optimizing a schedule for $n$ defendants requires $n^2$ integer variables, causing the size of the program to grow quadratically with the size of the scheduling pool. This scalability issue makes exact solutions impractical for large instances due to excessive memory and time requirements. 
Thus, an important aspect of our integrated optimization and learning framework, described in the next section, is to \emph{avoid solving} \eqref{model:OWA_group} directly directly.

\section{Optimization and Learning for Fair Court Schedules}
\label{sec:learning_with_OWA}

Problem \eqref{model:OWA_group} formalizes our fair scheduling problem as an optimization program that depends on unknown preference coefficients $\bm{Y}$. To address the absence of direct preference data, our approach involves learning these unknown preferences from available contextual information (e.g., the defendant features $\bm{x}_p$) using a neural network.  Recall from Section \ref{sec:Problem_Setting} that available training data consist of $\bm{D} = (\bm{x}_p, \bm{g}_p, \bm{Y}_p)_{p=1}^N$, where predictors $\bm{x}_p$ are known but $\bm{Y}_p$ are unknown at test time. 

A straightforward combined prediction and optimization model trains a neural network $\cM_{\theta}$ with weights $\theta$ to predict $\bm{Y}$ from $\bm{x}$, by minimizing the squared residuals of its predictions:
\begin{equation}
    \label{model:twostage}
    \min_{\theta} \frac{1}{N} \sum_p \left\| \cM_{\theta}(\bm{x}_{\bm{p}}) - \bm{Y}_{\bm{p}} \right\|^2.
\end{equation}

With this approach, problem \eqref{model:OWA_group} can be approximately specified by replacing $\bm{Y}$ with $\cM_{\theta}(\bm{x})$  in \eqref{model:OWA_group_obj}. It can then be solved assuming a suitable solution method. However, this ``two-stage'' approach faces two major challenges: 
\begin{enumerate}
    \item {\bf Scalability issues:} As discussed in the previous Section, problem \eqref{model:OWA_group} is an NP-hard nonlinear integer program whose size grows quadratically with the number of defendants. Thus an approach based on solving \eqref{model:OWA_group} will lack scalability.
    
    \item {\bf Misaligned training objective:} The training objective~\eqref{model:twostage} considers only the squared residuals of the \emph{predicted preference values}, rather than the utility~\eqref{model:OWA_group_obj} of the \emph{resulting downstream scheduling decisions}. This separation can lead to suboptimal scheduling outcomes because the model optimizes for prediction accuracy rather than the final scheduling utility. A substantial body of literature has demonstrated the benefits of training models to directly optimize the end-task objective, as discussed in Section~\ref{sec:related_work}.
\end{enumerate}

Motivated by these challenges, we propose an accurate and efficient alternative to the two-stage prediction and optimization method described above. Rather than training to minimize error in the preference values from their ground-truths as in \eqref{model:twostage}, our model is trained to \emph{directly maximize} the OWA value of predicted schedules under the ground-truth preference values. This requires computation of an optimal schedule as a function of preference values during each training iteration and performing backpropagation through the optimization process

\subsection{End-to-End Trainable Scheduling Model}
The proposed end-to-end trainable scheduling model consists of three main components: 
\begin{enumerate}
    \item A neural network $\cM_{\theta}$, which maps known features $\bm{x}_{p}$  to predicted preferences $\hat{\bm{Y}} = \cM_{\theta}(\bm{x}_{p})   $.
    \item A differentiable module $\bm{\Pi}$ that maps $\hat{\bm{Y}}$ to a \emph{predicted permutation matrix} $\bm{\Pi}(\hat{\bm{Y}})$, which satisfies constraints \eqref{model:OWA_group_individual_integers}-\eqref{model:OWA_group_individual_constraints_j}.
    \item A loss function which allows training of $\cM_{\theta}$ to optimize the objective    $\textsc{OWA}_{\mathbf{w}} \left( \; \bm{u}^{\cG}(\;\;\bm{\Pi}(\hat{\bm{Y}}),\bm{Y}_p\;\;) \right)$ expressed in~\eqref{model:OWA_group_obj}  by gradient descent.
\end{enumerate}

Composed of these elements, the resulting end-to-end ML and optimization training objective is: 
\begin{equation}
    \label{model:endtoend_training}
    \max_{\theta} \frac{1}{N} \sum_p \textsc{OWA}_{\mathbf{w}} 
    \left( \bm{u}^{\cG}\left( \bm{\Pi}(\cM_{\theta}(\bm{x}_{p})),\bm{Y}_p\right) \right).
\end{equation}

A comparison of \eqref{model:endtoend_training} with \eqref{model:OWA_group} highlights the motivation for this architecture: Gradient descent on the empirical objective \eqref{model:endtoend_training} causes $\cM_{\theta}$ to learn predictions $\hat{\bm{Y}}$ which yield the permutations $\bm{\Pi}(\hat{\bm{Y}})$ that solve \eqref{model:OWA_group} for $\bm{\Pi}^{\star}(\bm{Y}_p)$, given features $\bm{x}_p$. 
Thus, the resulting composite mapping $\bm{\Pi} \circ \cM_{\theta}$ fulfills our goal of a model which predicts schedules that solve the fair scheduling problem \eqref{model:OWA_group} \emph{without direct knowledge of preference data $\bm{Y}$}! 

What remains is to determine efficient and \emph{differentiable} implementations of the module $\bm{\Pi}$ and loss function $\textsc{OWA}_{\mathbf{w}}$, to enable backpropagation for gradient descent training of \eqref{model:endtoend_training}.

\subsection{Differentiable Matching Layer}
Our proposed fair scheduling model relies on a module $\bm{\Pi}(\bm{Y})$ which maps preference data to permutation matrices. Note that \eqref{model:OWA_group} is one such mapping. However, it is neither efficient to compute nor differentiable, thus unsuitable for gradient descent training \eqref{model:endtoend_training}. We address the efficiency aspect first, noting that the OWA objective is not required to yield a valid permutation. 

Instead, consider replacing the OWA objective in \eqref{model:OWA_group} with the total utility objective \eqref{eq:utility}:
\begin{subequations}
    \label{model:matching_layer}
    \begin{align}
       \bm{\Pi}(\bm{Y}) = {\argmax}_{\bm{\Pi}} &\;\; \mathrm{Tr}\left(\bm{Y}^T \bm{\Pi}\right) \label{model:matching_layer_obj} \\
       \mbox{subject to:} & \;\;\;\bm{\Pi} \in \{0,1\}^{n \times n} \label{model:matching_layer_integers} \\
       & \;\; \sum_{i}  \Pi_{ij} =1, \forall j \in [n] \label{model:matching_layer_constraints_i} \\
       & \;\; \sum_{j} \Pi_{ij}=1, \forall i \in [n]. \label{model:matching_layer_constraints_j}
    \end{align}
\end{subequations}
As the sum of individual utilities, the objective \eqref{model:matching_layer_obj} is linear. Additionlly, we identify the constraints \eqref{model:matching_layer_constraints_j}-\eqref{model:matching_layer_constraints_i} as being \emph{totally unimodular} with integer right-hand side. This means that when arranged in the form $\bm{A}\bm{\Pi} = \bm{b}$, the vector $\bm{b}$ is integer while every square submatrix of $\bm{A}$ has determinant $0$ or $1$. A well-known result in optimization states that the optimal solution to such a linear program (LP) must be integer-valued \cite{bazaraa2008linear}. 

Thus, we may solve \eqref{model:matching_layer} by replacing \eqref{model:matching_layer_integers} with continuous bounds $\bm{\Pi} \in [0,1]^{n \times n}$ and solving the resulting LP. In fact, this LP is recognized as the classic \emph{assignment problem}, which admits known solutions in $\cO(n^3)$ time \cite{bazaraa2008linear}. We refer to the mapping \eqref{model:matching_layer} as the \emph{matching layer}.

\begin{proposition}
    The matching layer \eqref{model:matching_layer} has complexity $\cO(n^3)$.
\end{proposition}

\begin{proof} 
The matching layer problem \eqref{model:matching_layer} is equivalent to the assignment problem, which can be solved in $\mathcal{O}(n^3)$ time using algorithms such as the Hungarian method \cite{Kuhn1955Hungarian}. 

Specifically, the cost matrix in the assignment problem is given by the negative of the preference matrix $\bm{Y}$, and the goal is to find the permutation matrix $\bm{\Pi}$ that maximizes the total utility (or equivalently, minimizes the total cost). Therefore $\bm{\Pi}(\bm{Y})$ can be computed in cubic time with respect to the number of defendants $n$. 
\end{proof}

It remains to implement differentiation for backpropagation in training of \eqref{model:endtoend_training}. Differentiation of linear programs has recently been thoroughly studied \cite{Fioretto:jair24}, and one of various proposals may be employed. We choose a method \cite{vlastelica2020differentiation} which approximates gradients through an LP problem by a finite difference between two solutions of \eqref{model:matching_layer} under perturbed inputs $\bm{Y}$. Thus the combined forward and backward passes through   \eqref{model:matching_layer} are assured $\cO(n^3)$ complexity, yielding an efficient and differentiable matching layer.

\subsection{OWA as a Loss Function}
To enable gradient descent training for \eqref{model:endtoend_training}, the final element to be specified is a method for backpropagating the OWA loss function. Although the OWA is not differentiable, it is \emph{subdifferentiable} with known subgradients \cite{do2022optimizing}:
\begin{equation}
    \label{eq:owa_subgrad}
    \frac{\partial}{\partial \bm{x}} \textsc{OWA}_{\bm{w}}(\bm{x}) = \bm{w}_{(\sigma^{-1})},
\end{equation}
where $\sigma$ is the sorting permutation for $\bm{x}$. Our main approach is to implement subgradient descent training for \eqref{model:endtoend_training} using the formula   \eqref{eq:owa_subgrad}. However, we also investigate the use of an alternative gradient rule in this paper. For any convex function $f$, the \emph{Moreau envelope} $f^{\beta}$ is a $\frac{1}{\beta}$ smooth lower-bounding function with the same minimum: 
\begin{equation}
    \label{eq:moreau_env}
    f^{\beta}(\mathbf{x}) = \min_{\mathbf{v}} \;\;  f(\mathbf{v}) + \frac{1}{2 \beta} \|   \mathbf{v} - \mathbf{x}  \|^2.
\end{equation}

It is proven in \cite{do2022optimizing} that the gradient of OWA's Moreau envelope is equal to the projection of a vector $\mathbf{x}$ onto the permutahedron $\cC(\tilde{\bm{w}})$ induced by the OWA weights $\bm{w}$:
\begin{equation}
    \label{eq:moreau_projection}
    \frac{\partial}{\partial \bm{x}} \textsc{OWA}_{\bm{w}}^{\beta}(\mathbf{x}) = \proj_{\cC(\tilde{\bm{w}})} \left(\frac{\bm{x}}{\beta}\right).
\end{equation}
Thus in Section \ref{sec:Experiments}, we evaluate an additional training of \eqref{model:endtoend_training}, using OWA's Moreau envelope as a differentiable loss function. We implement \eqref{eq:moreau_projection} using isotonic regression following \cite{blondel2020fast}. Both \eqref{eq:owa_subgrad} and \eqref{eq:moreau_projection} have the same computational complexity, driven by the sorting operation, which is $\cO(n \log n)$.

\section{Experimental Settings}
\label{sec:Experiments}
To validate the effectiveness of our proposed integrated optimization and learning framework for fair court scheduling, we conduct a series of experiments on court scheduling using a novel synthetic dataset. 
We first focus on the approach adopted to generate the synthetic preference data for court scheduling.

\subsection{Data Generation Process}

The experiments simulate realistic court scheduling scenarios by creating a causal graph that models the data generation process, incorporating individuals' preferences, socioeconomic status, and demographic characteristics. This causal graph, depicted in Figure \ref{fig:causalgraph_court}, illustrates the relationships among these factors. 
Each feature is treated as a categorical variable, generated based on the conditional probabilities detailed in Appendix \ref{app:prob_tab}, which are elicited by combining, interviewing court scheduling experts with Census data. 

To emulate a scenario requiring fair scheduling, we generate individuals' preferences to predominantly favor specific time slots. For example, individuals reliant on public transportation, those working night shifts, or those with childcare responsibilities are more likely to prefer morning slots (8-12 AM). Each individual's preference vector $\bm{y}_i$ is generated under the constraints that $\sum_{j} y_{ij} =1$ and $0 \leq y_{ij} \leq 1$. Additionally, we assume that each individual has three top choices among all possible time slots, ranked from highest to lowest priority. The second and third choices are set to one hour before and one hour after the primary choice, respectively. This setup ensures that when specific time slots have limited availability, individuals' second or third preferences are considered, promoting fairness and flexibility in scheduling.

\begin{figure}[th]
\centerline{\includegraphics[width=0.6\columnwidth]{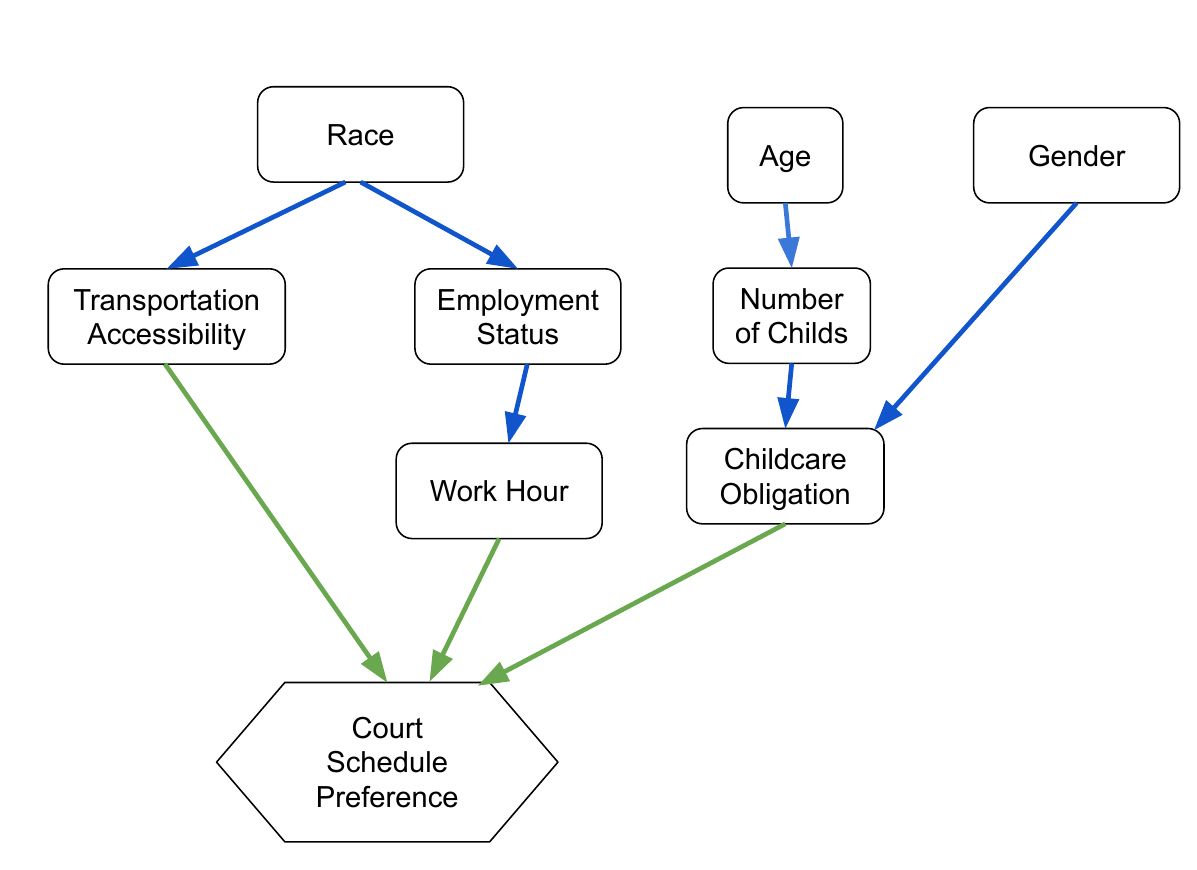}}
\caption{Causal relationship depicts factors affecting individual's court schedule preference. Blue arrows indicate indirect relationships, while green arrows indicate direct relationships. }
\label{fig:causalgraph_court}
\end{figure}

We generate different training sets with varying pool sizes of defendants, specifically $N \in \{25, 50, 100, 250, 500, 4000\}$, where each pool consists of $n=12$ individuals assigned to corresponding time slots. Each individual is allocated to exactly one slot. The paper evaluates the proposed framework based on both group and individual fairness notions (where the group size is 1 for individual fairness). For group fairness evaluation, the paper creates different settings using protected group attributes such as transportation accessibility, employment status, and work hours. The performance is evaluated on the same test set containing $N = 500$ samples, generated using the same probability distribution as the training data. More details about the probability distributions of both training and test sets can be found in Appendix \eqref{app:prob_tab}.

\subsection{Model Settings and Evaluation Metrics}

\begin{figure*}[th]
\centerline{\includegraphics[width=1.\columnwidth]{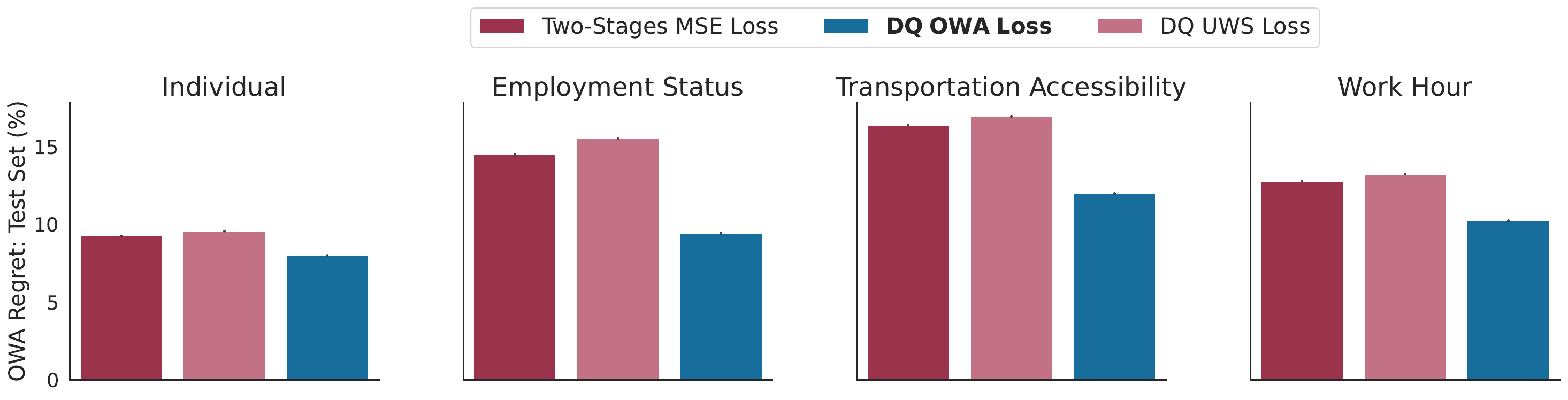}}
\caption{OWA utility regret \eqref{eq:regret} in court scheduling across different fairness levels: individual (first subplot) vs. group (last three subplot). Protected group attributes are employment status, transportation accessibility, and work hours. }
\label{fig:results_court}
\end{figure*}

\paragraph{Settings.}
A feedforward neural network $\cM_{\theta}$ is trained to predict for a pool of $n$ candidates, given features $\bm{x}$, their preference scores $\bm{Y} \in \mathbb{R}^{n\times n}$ for $n$ available slots. The network consists of two hidden layers, where the size of each successive layer is halved. The model is trained using the Adam optimizer, with a learning rate of 0.01 and a batch size $\in$ \{64, 128, 512\}. Results for each hyperparameter setting are taken on average over five random seeds.

\begin{figure*}[t]
\centerline{\includegraphics[width=1.\columnwidth]{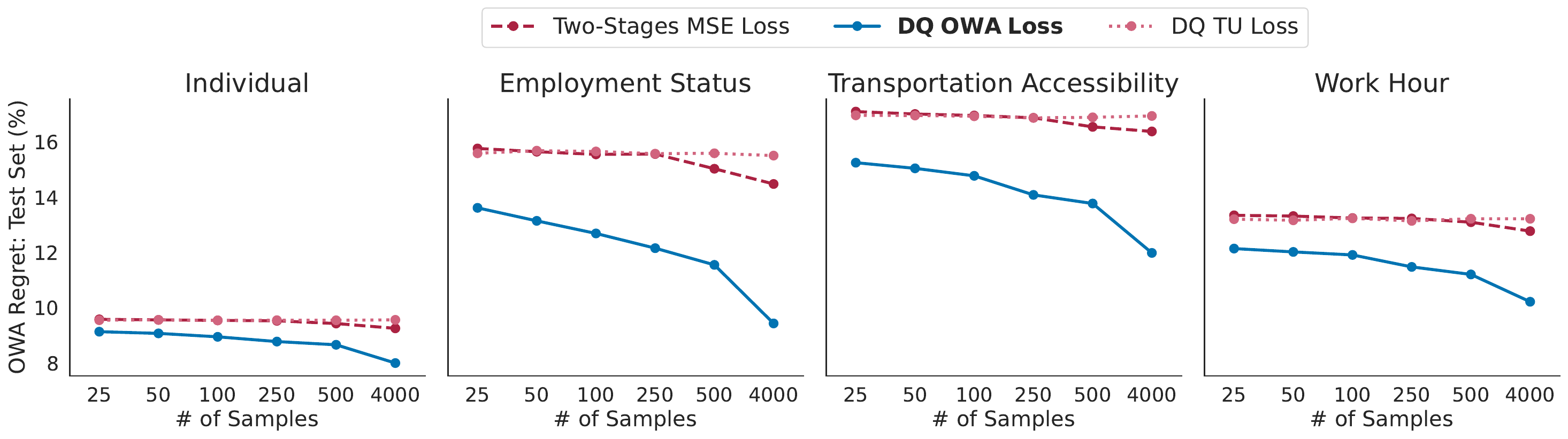}}
\caption{Benchmarking OWA Regret (in percentage) across different training data sizes.}
\label{fig:results_sparse}
\end{figure*}

\paragraph{Evaluation metrics.}
All the experiments are evaluated primarily using two key metrics: \emph{regret} and \emph{normalized pairwise distances}. 

\smallskip\noindent\textbf{Regret:} 
Regret quantifies the loss of optimality in the obtained schedule due to prediction errors in the estimated preferences $\hat{\bm{Y}}$. Specifically, it measures how much worse the OWA objective value is when using the predicted preferences compared to using the true preferences. Formally, regret is defined as:
\begin{align}
    \label{eq:regret}
    \textit{regret}(\hat{\bm{Y}},\bm{Y}) &= \textsc{OWA}_{\mathbf{w}} \left( \; \bm{u}^{\cG}(\bm{\Pi}^{\star}(\bm{Y}),\bm{Y}) \;\right) \notag \\ 
    &- \textsc{OWA}_{\mathbf{w}} \left( \; \bm{u}^{\cG}(\bm{\Pi}^{\star}(\bm{\hat{Y}}),\bm{Y}) \;\right),
\end{align}
where:
$\bm{Y}$ represents the ground-truth preference matrix,
$\hat{\bm{Y}}$ is the predicted preference matrix obtained from the model $\cM_{\theta}$,
$\bm{\Pi}^{\star}(\bm{Y})$ is the optimal schedule computed using the true preferences,
$\bm{\Pi}^{\star}(\hat{\bm{Y}})$ is the schedule computed using the predicted preferences,
$\bm{u}^{\cG}$ computes the utility vector for the group $\cG$ under the given schedule and preferences.
A lower regret value indicates that the predicted schedule is closer to the optimal one, with a regret of zero signifying that the prediction $\hat{\bm{Y}}$ leads to the optimal schedule under the true preferences $\bm{Y}$. Minimizing regret is crucial, as it reflects the accuracy of the scheduling decisions in terms of the OWA objective.

\smallskip\noindent\textbf{Normalized pairwise difference:} 
The normalized pairwise difference (NMPD) is an intuitive measure of fairness across individuals. It assesses how similarly individuals are treated by comparing the differences in their outcomes. Formally, for a set of individuals with outcomes $(u_1, \ldots, u_n)$, the NMPD is defined as:
\begin{equation}
    \label{eq:gini}
    \textit{NMPD}(\bm{u}) = \frac{1}{n^2 \bar{\bm{u}}} \sum_{i=1}^n\sum_{j=1}^n |u_i - u_j| \; \;\;\;\mbox{with} \; \; \bar{\bm{u}} = \frac{1}{n}\sum_{i=1}^n u_i
\end{equation}
The normalized mean pairwise difference is valuable as a fairness metric because it reflects how similarly or differently individuals with analogous characteristics are treated. A lower NMPD value indicates that the outcomes are more uniform across individuals, suggesting a fairer distribution of utility. Conversely, a higher NMPD signifies greater disparities in individual outcomes, pointing to potential biases or unfairness in the scheduling process. Therefore, minimizing NMPD is important for achieving equitable treatment of all defendants.

\subsection{Baseline Models}

To evaluate the effectiveness of our proposed model, we compare it against two baseline methods:

\begin{enumerate}
\item \textbf{Two-Stage Method}: 
This standard baseline in Predict-Then-Optimize frameworks, as discussed by \cite{mandi2023decision}, involves training the prediction model $\hat{\bm{Y}} = \cM_{\theta}(\bm{x})$ using mean squared error (MSE) regression. The training objective is to minimize the loss:
   \[
   \mathcal{L}_{\text{TS}}(\hat{\bm{Y}}, \bm{Y}) = \| \hat{\bm{Y}} - \bm{Y} \|^2.
   \]
   The downstream optimization model is employed only at test time and is not considered during training. This approach may not account for the impact of prediction errors on the final scheduling utility.

\item \textbf{Total Utility (TU) Loss}: This method employs an end-to-end learning approach using a differentiable matching layer but optimizes for the total sum of utilities rather than incorporating the Ordered Weighted Average (OWA) objective. While this method benefits from integrating the downstream optimization into the learning process, it may not fully capture fairness considerations inherent in the OWA objective.
\end{enumerate}

\section{Results}

In this section, we present the results of our experiments evaluating the performance of our proposed integrated optimization and learning framework for fair court scheduling. We compare the performance of three models in each evaluation setting:
\begin{enumerate}
\item \textbf{Two-Stage MSE Loss}: The traditional two-stage model trained using Mean Squared Error (MSE) loss without considering downstream optimization.
\item \textbf{OWA Loss Decision Quality (DQ)}: Our proposed end-to-end model trained to directly optimize the OWA objective, integrating the prediction and optimization stages.
\item \textbf{Total Utility (TU) Loss}: An end-to-end model trained to maximize the total utility without incorporating the OWA fairness considerations.
\end{enumerate}

\subsection{OWA Utility Regret}
Figure \ref{fig:results_court} illustrates the OWA regret (expressed as a percentage) across four different fairness settings: individual fairness, employment status, transportation accessibility, and work hours. The evaluation is conducted on a training dataset consisting of $N=4000$ samples. The OWA regret measures the loss of optimality relative to the ground truth preferences, with lower values indicating better alignment with the optimal solution. 
\begin{figure*}[t]
\centerline{\includegraphics[width=1.\columnwidth]{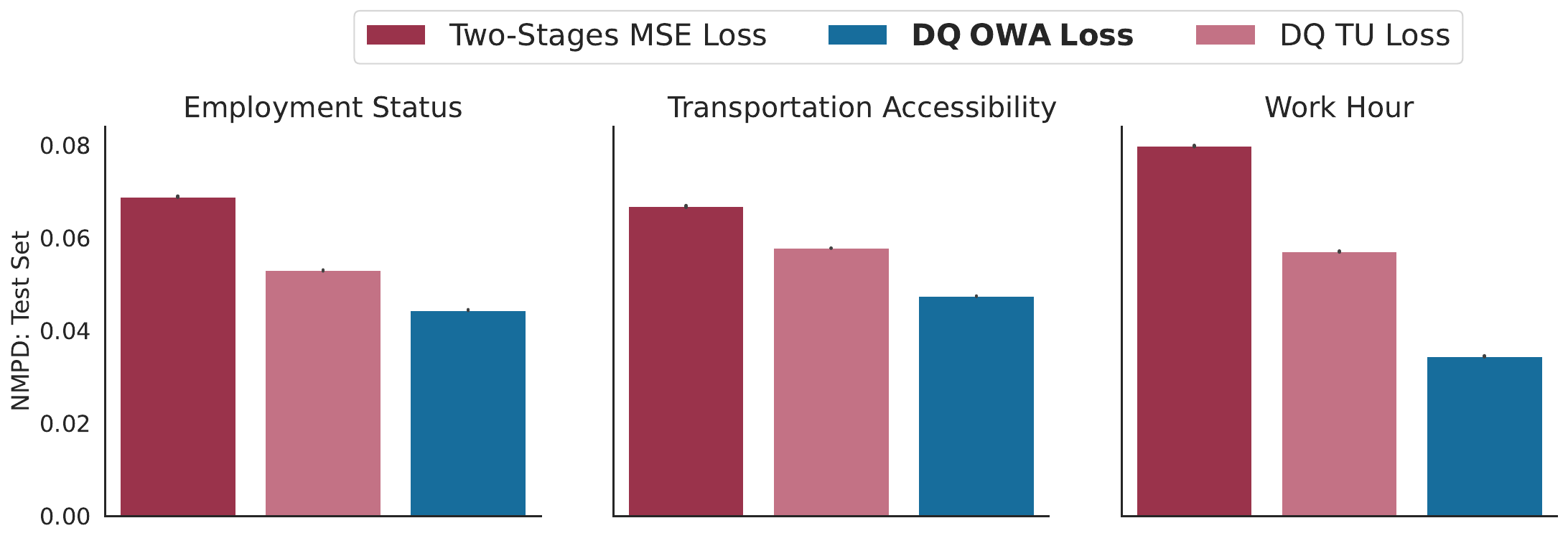}}
\caption{Normalized mean pairwise difference \eqref{eq:gini} in court scheduling across different fairness levels: individual (first subplot) vs. group (last three subplot). Protected group attributes are employment status, transportation accessibility, and work hours. }
\label{fig:results_court_mpd}
\end{figure*}

Across all fairness settings, the {\bf OWA Loss DQ} model consistently achieves the lowest regret values, indicating superior performance in minimizing the loss of optimality. Specifically, it attains regret values of less than 10\% in the individual fairness and employment status settings, and 11\% and 12.5\% in the work hours and transportation accessibility settings, respectively.

The \textbf{Two-Stage MSE Loss} and \textbf{TU Loss} models exhibit similar performance, with the Two-Stage method slightly outperforming the TU Loss method in certain settings. For instance, the Two-Stage model shows improvements of 0.9\% and 0.4\% over the TU Loss model in the employment status and transportation accessibility contexts, respectively.

Generally, end-to-end learning approaches, such as the {\bf OWA Loss DQ} and  \textbf{TU Loss} models, tend to deliver better performance than two-stage methods, as they integrate prediction and optimization into a unified framework. However, when the training dataset is sufficiently large, the \textbf{Two-Stage MSE Loss} approach demonstrates its capability to effectively minimize decision quality loss due to the abundance of data improving prediction accuracy.

Notably, our {\bf OWA Loss DQ} model performs significantly better in the employment status and transportation accessibility settings, achieving improvements of 7\% and 5.2\%, respectively, compared to the \textbf{TU Loss} model. Furthermore, the  {\bf OWA Loss DQ} outperforms the \textbf{Two-Stage MSE Loss} method by margins of 6.1\% and 4\% in these settings.

In contrast, the performance gains in the individual fairness and work hours settings are smaller, with improvements of only 1.5\% and 3.6\%, respectively. This variation in performance suggests that the effectiveness of the models may depend on the characteristics of the dataset and the specific fairness constraints involved.

When the group partitions are such that there are no competing demands among groups—that is, when the preferences of different groups do not conflict—the solutions tend to be nearly optimal regardless of the model used. This observation highlights the importance of understanding the underlying data structure and the dynamics of group preferences when evaluating fairness in scheduling.

Figure \ref{fig:results_sparse} illustrates the performance of all models in terms of OWA regret on the test set across varying training data sizes. The plot shows regret percentages for test sets across different sample sizes (25, 50, 100, 250, 500, 4000) and fairness settings: individual fairness, employment status, transportation accessibility, and work hours. The y-axis represents the regret percentage, while the x-axis shows the number of samples used. For each setting, the regret decreases as the number of samples increases, indicating that the model performs better with more data. Different fairness settings exhibit varying regret levels, matching the trend demonstrated in Figure \ref{fig:results_court}. For example, the transportation accessibility setting consistently shows higher regret compared to others, suggesting more difficulty in optimizing solutions effectively. In contrast, the employment status setting shows a significant drop in regret as training size increases, aligning with previous observations that dataset biases impact model learning. There is a small drop in regret when sample sizes increase from 25 to 100, with a more noticeable decrease between 250–500 samples. Both \textbf{Two-Stage MSE Loss} and  \textbf{TU Loss}  models perform worse with fewer samples and show only small improvements as the sample size grows. In contrast, our {\bf OWA Loss DQ} model not only has a steeper learning curve but also outperforms other models—even with small datasets (N=25)—achieving regret percentages of 9.2\%, 13.8\%, 14.3\%, and 12.1\% for individual, employment status, transportation accessibility, and work hour settings, respectively. This highlights our model's strength in operating on small datasets, which is particularly valuable when preference datasets for court systems are difficult to obtain.

\subsection{Normalized Pairwise Difference}

To comprehensively evaluate the fairness of our scheduling models, we utilize the Normalized Mean Pairwise Difference (NMPD) metric. Figure \ref{fig:results_court_mpd}
presents the NMPD values across various fairness settings, including employment status, transportation accessibility, and work hours. 

\paragraph{Fairness performance across settings.} 
Across all fairness settings, the NMPD values remain relatively low, typically below 0.08. This indicates that the models maintain a high degree of fairness by ensuring consistent treatment of individuals, with minimal differences in outcomes. Specifically:
\begin{itemize}
    \item {\it End-to-End Models}: Both the  {\bf OWA Loss DQ} and {\bf TU Loss} models exhibit consistently low NMPD values across all settings, demonstrating their effectiveness in promoting equitable treatment. Our proposed framework, {\bf OWA Loss DQ} model, consistently outperforms the  {\bf TU Loss} model, achieving lower NMPD values and thereby ensuring greater fairness.
    
    \item {\it Two-Stage MSE Loss}: This traditional approach shows slightly higher NMPD values, particularly in the work hours and transportation accessibility settings. This suggests that the two-stage method is less effective in maintaining fairness in these areas compared to the end-to-end models.
\end{itemize}

\begin{figure}[ht]
    \centering
    \includegraphics[width=0.6\columnwidth]{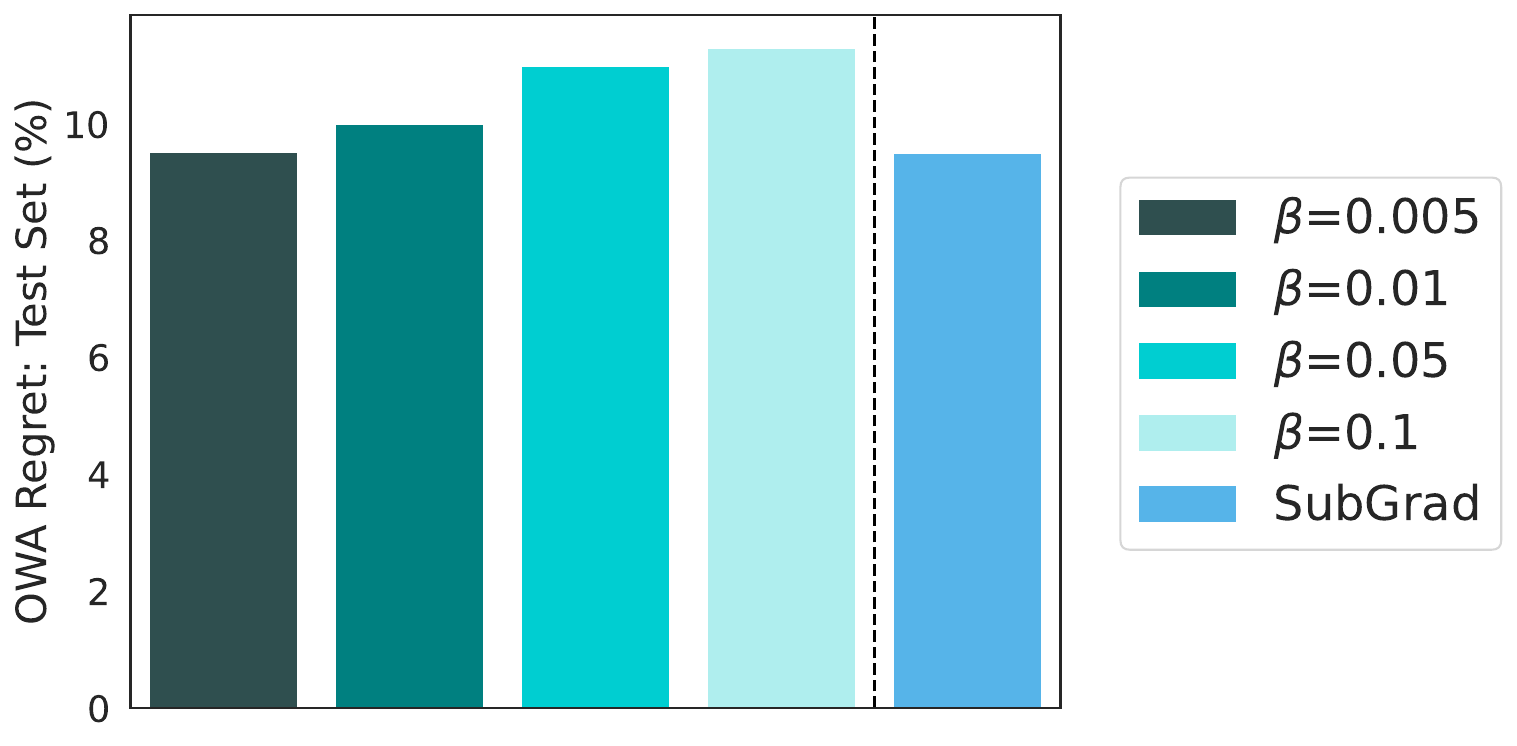}
    \caption{Effect of training OWA DQ Loss with Subgradients and Moreau envelope gradient. The first four bars represent OWA DQ loss with Moreau OWA's Gradient under various smoothing hyperparameters, while the leftmost bar shows OWA DQ loss with subgradients. Results reported on employment status setting.  }
    \label{fig:subgrad}
\end{figure}

\paragraph{Impact of gradient-based methods.}
Figure \ref{fig:subgrad} compares the effect of training OWA loss with these two gradient-based methods. The smoothness of the Moreau envelope, controlled by the hyperparameter $\beta$, is reflected in the performance shown by the four leftmost bars in the figure.  Smaller values of $\beta$ result in improved performance, with $\beta = 0.005$ yielding results comparable to the subgradient method. These findings align with the theoretical asymptotic behavior \cite{blondel2020fast}. Therefore, for simplicity, we adopt the subgradient model $\bm{g}$ in \eqref{eq:owa_subgrad} as our preferred approach for the remainder of the paper.

\smallskip

These experiments reveal that the  {\bf OWA Loss DQ} model consistently achieves the lowest NMPD values across all fairness settings, thus ensuring higher fair decision-making. 
In particular, the low NMPD values across models, especially within end-to-end approaches, highlight the effectiveness of integrating fairness objectives directly into the learning and optimization processes. This is an important consideration as two-stage methods falls short in certain settings. These results validate the importance of incorporating fairness directly into the model training process and demonstrate the effectiveness of our integrated optimization and learning framework in achieving equitable and efficient court scheduling outcomes.

\begin{figure}[h]
    \centering
    \includegraphics[width=0.6\columnwidth]{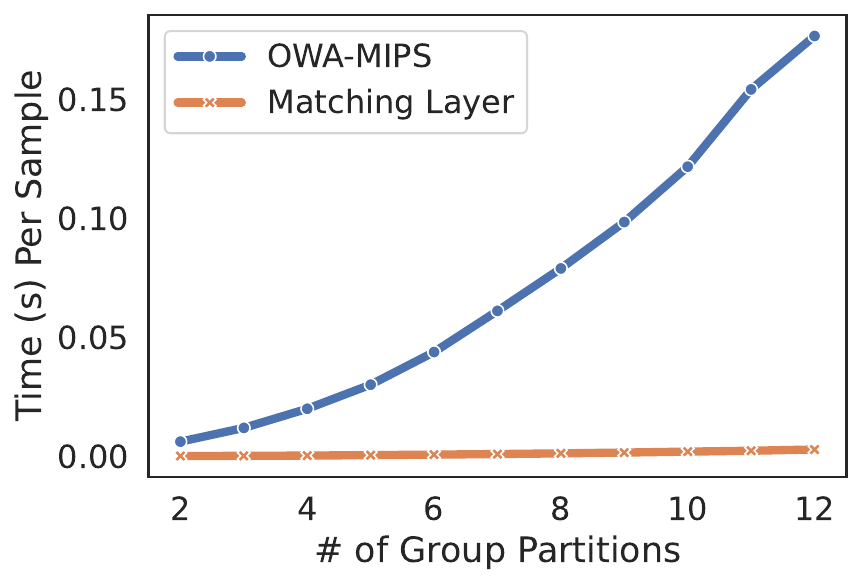}
    \caption{Benchmarking the runtime of two optimization models: OWA-ILPs used in the \textbf{Two-Stage MSE Loss} model during inference, and the Matching Layer employed in the  {\bf OWA Loss DQ} model.}
    \label{fig:time_complexity}
\end{figure}
\subsection{Running Time}

Figure \ref{fig:time_complexity} presents the runtime comparison of two optimization models for the court scheduling problem: \textit{OWA-ILPs}  corresponding to Problem \eqref{model:OWA_group}, and the\textit{Matching Layer}  representing Problem \eqref{model:matching_layer}. 

The x-axis denotes the number of group partitions, while the y-axis indicates the average time (in seconds) required to solve the scheduling problem for each sample pool, averaged over $N = 1000$ runs.
As detailed in Section \ref{sec:OWA_optimization}, solving Problem \eqref{model:OWA_group}
involves addressing an Integer Linear Program (ILP). ILPs are computationally intensive and face significant scalability challenges as the number of group partitions increases due to the exponential growth of constraints. This scalability issue is clearly illustrated in Figure \ref{fig:time_complexity}, where the runtime for OWA-ILPs grows polynomially with the number of partitions, rendering it impractical for large-scale scheduling tasks.

In contrast, the Matching Layer uses a linear optimization approach, dramatically improving computational efficiency. The runtime for the Matching Layer scales linearly with the number of partitions, as shown in Figure \ref{fig:time_complexity}. For example, when set to 12 partitions, the Matching Layer solves the scheduling problem in just 0.002 seconds. This linear scalability makes the Matching Layer ideal for real-time and large-scale court scheduling applications where quick decision-making is essential.

The stark difference in runtime performance between the two models highlights the technical advantage of the Matching Layer over traditional ILP-based approaches. By leveraging linear optimization techniques and exploiting problem-specific structures, the Matching Layer not only reduces computational burden but also facilitates the integration of fairness objectives into the scheduling process without sacrificing efficiency. This makes it a highly effective solution for practical court scheduling systems that require both fairness and scalability.



\section{Conclusions}
This paper addressed the critical problem of fair court scheduling in pretrial processes. Scheduling defendants' court appearances in a manner that is both efficient and fair is essential for upholding justice and maintaining public trust in the legal system. The challenge lies in aligning court schedules with defendants' preferences, which are often unknown and must be predicted from available data, while ensuring that the scheduling process adheres to fairness principles to prevent systemic biases and disparate impacts on different groups.

To address these key challenges this paper proposed an integrated optimization and learning framework that combines machine learning with Fair Ordered Weighted Average (OWA) optimization. Unlike traditional two-stage approaches that handle prediction and optimization separately, our method integrates the prediction of defendants' preferences with the scheduling optimization process. This integration allows for direct optimization of scheduling utility under fairness constraints, leading to more equitable and efficient scheduling outcomes.

Through extensive experiments, we demonstrated that our integrated framework outperforms baseline models in terms of both scheduling optimality and fairness across various scenarios. Our results highlight the effectiveness of incorporating fairness objectives into the learning process, particularly in complex settings with competing group preferences.
We believe that this work could pave the way for the utilization of Fair OWA in learning pipelines, enabling a wide range of critical multi-optimization problems across various  domains that extend beyond scheduling applications.
\section*{Acknowledgement}
This research is partially supported by NSF grants 2232054, 2232055, 2133169, and NSF CAREER Award 2143706. The views and conclusions of this work are those of the authors only.
\bibliographystyle{unsrt}  
\bibliography{main_arxiv}

\newpage\appendix

\section{Causal Graph Conditional Probability Tables}
\label{app:prob_tab}

Below are the conditional distributions of the causal graph depicted in Figure \ref{fig:causalgraph_court}. 

\subsection*{Distribution of Race among Defendants}
\centering 
\begin{tabular}{lc} 
\toprule \textbf{Race} & \textbf{P(x)} \\ 
\midrule 
White & 0.5 \\ 
Non White & 0.5 \\ 
\bottomrule 
\end{tabular} 

\subsection*{Distribution of Age Groups among Defendants}
\centering 
\begin{tabular}{lc} 
\toprule 
\textbf{Age Group} & \textbf{P(x)} \\ 
\midrule 
Below 18 & 0.05 \\ 
18-54 & 0.8 \\ 
Above 55 & 0.15 \\ 
\bottomrule 
\end{tabular} 

\subsection*{Distribution of Gender among Defendants}
\centering 
\begin{tabular}{lc} 
\toprule \textbf{Gender} & \textbf{P(x)} \\ 
\midrule Male & 0.45 \\ Female & 0.55 \\ 
\bottomrule 
\end{tabular}

\subsection*{Transportation Accessibility by Race}

\centering \begin{tabular}{llc} \toprule 
\textbf{Transportation Accessibility (x)} & \textbf{Race (y)} & \textbf{P(x|y)} \\ 
\midrule Public transportation & White & 0.8 \\ 
Private transportation & White & 0.2 \\ 
Public transportation & Non White & 0.6 \\ 
Private transportation & Non White & 0.4 \\ 
\bottomrule \end{tabular} 

\subsection*{Employment Status by Race}

\centering \begin{tabular}{llc} \toprule 
\textbf{Employment Status (x)} & \textbf{Race (y)} & \textbf{P(x|y)} \\
\midrule 
Employed & White & 0.8 \\ 
Unemployed & White & 0.2 \\ 
Employed & Non White & 0.7 \\ 
Unemployed & Non White & 0.3 \\ 
\bottomrule \end{tabular} 

\subsection*{Work Hour by Employment Status}
\begin{tabular}{llc}
    \toprule
    \textbf{Work Hour} & \textbf{Employment Status} & \textbf{P(x|y)} \\
    \midrule
    Day shift & Employed & 0.5 \\
    Night shift & Employed & 0.3 \\
    Irregular shift & Employed & 0.18 \\
    No shift & Employed & 0.02 \\
    Day shift & Unemployed & 0.0 \\
    Night shift & Unemployed & 0.0 \\
    Irregular shift & Unemployed & 0.0 \\
    No shift & Unemployed & 1.0 \\
    \bottomrule
\end{tabular}
 
\subsection*{Number of Children by Age Group}
\centering \begin{tabular}{llc} \toprule 
\textbf{Number of Children} & \textbf{Age Group} & \textbf{P(x|y)} \\
 \midrule No child & Under 18 & 0.95 \\
 +1 child & Under 18 & 0.05 \\
 No child & 18-54 & 0.55 \\
 +1 child & 18-54 & 0.45 \\
 No child & Above 55 & 0.2 \\
 +1 child & Above 55 & 0.8 \\
 \bottomrule 
 \end{tabular}

\subsection*{Childcare Obligation by Gender and Number of Children}

\centering 
\begin{tabular}{llc} \toprule 
\textbf{Childcare Obligation} & \textbf{Gender, Number of Children} & \textbf{P(x|y,z)} \\
\midrule No obligation & Female, no child & 1.0 \\
No obligation & Female, +1 child & 0.3 \\
No obligation & Male, no child & 1.0 \\
No obligation & Male, +1 child & 0.85 \\
Have obligation & Female, no child & 0.0 \\
Have obligation & Female, +1 child & 0.7 \\
Have obligation & Male, no child & 0.0 \\
Have obligation & Male, +1 child & 0.15 \\
\bottomrule 
\end{tabular}

\subsection*{Schedule Preference by Transportation Accessibility, Work Hour, and Childcare Obligation}

\begin{table*}[ht]
\centering
\resizebox{0.6\textwidth}{!}{%
\begin{tabular}{llc} 
    \toprule
    \textbf{Schedule Preference (o) } & \textbf{Transportation Accessibility (l), Work Hour (m), Childcare Obligation (n)} & \textbf{P(o | l,m,n)} \\
    \midrule
    \multicolumn{2}{l}{\textbf{Public Transportation, Day or Regular Shift, Have or Don't Have Childcare Obligation}} & \\
    8:00AM & & 1/6 \\
    8:30AM & & 1/6 \\
    9:00AM & & 1/6 \\
    9:30AM & & 1/6 \\
    10:00AM & & 1/6 \\
    10:30AM & & 1/6 \\
    1:00PM & & 0 \\
    1:30PM & & 0 \\
    2:00PM & & 0 \\
    2:30PM & & 0 \\
    3:00PM & & 0 \\
    3:30PM & & 0 \\
    \midrule
    \multicolumn{2}{l}{\textbf{Public Transportation, Night Shift, Have or Don't Have Childcare Obligation}} & \\
    8:00AM & & 0 \\
    8:30AM & & 0 \\
    9:00AM & & 0 \\
    9:30AM & & 1/3 \\
    10:00AM & & 1/3 \\
    10:30AM & & 1/3 \\
    1:00PM & & 0 \\
    1:30PM & & 0 \\
    2:00PM & & 0 \\
    2:30PM & & 0 \\
    3:00PM & & 0 \\
    3:30PM & & 0 \\
    \midrule
    \multicolumn{2}{l}{\textbf{Private Transportation, Day Shift, Have Childcare Obligation}} & \\
    8:00AM & & 1/6 \\
    8:30AM & & 1/6 \\
    9:00AM & & 1/6 \\
    9:30AM & & 1/6 \\
    10:00AM & & 1/6 \\
    10:30AM & & 1/6 \\
    1:00PM & & 0 \\
    1:30PM & & 0 \\
    2:00PM & & 0 \\
    2:30PM & & 0 \\
    3:00PM & & 0 \\
    3:30PM & & 0 \\
    \midrule
    \multicolumn{2}{l}{\textbf{Private Transportation, Day Shift, No Childcare Obligation}} & \\
    8:00AM & & 0 \\
    8:30AM & & 0 \\
    9:00AM & & 0 \\
    9:30AM & & 0 \\
    10:00AM & & 0 \\
    10:30AM & & 0 \\
    1:00PM & & 0 \\
    1:30PM & & 0 \\
    2:00PM & & 0 \\
    2:30PM & & 1/3 \\
    3:00PM & & 1/3 \\
    3:30PM & & 1/3 \\
    \midrule
    \multicolumn{2}{l}{\textbf{Private Transportation, Night or Irregular Shift, Have or Don't Have Childcare Obligation}} & \\
    8:00AM & & 1/4 \\
    8:30AM & & 1/4 \\
    9:00AM & & 1/4 \\
    9:30AM & & 1/4 \\
    10:00AM & & 0 \\
    10:30AM & & 0 \\
    1:00PM & & 0 \\
    1:30PM & & 0 \\
    2:00PM & & 0 \\
    2:30PM & & 0 \\
    3:00PM & & 0 \\
    3:30PM & & 0 \\
    \bottomrule
\end{tabular}%
}
\caption{Schedule preferences based on transportation accessibility, work hours, and childcare obligation with corresponding probabilities \( P(o|l,m,n) \).}
\label{table:schedule_preferences}
\end{table*}

\end{document}